\documentclass{article}

\usepackage{microtype}
\usepackage{graphicx}
\usepackage{subcaption}
\usepackage{booktabs} 
\usepackage{placeins}   

\usepackage{hyperref}

\usepackage{amsmath,amsfonts,bm}

\def\eqref#1{equation~\ref{#1}}

\def\1{\bm{1}}

\def\rvx{{\mathbf{x}}}
\def\rvy{{\mathbf{y}}}

\DeclareMathAlphabet{\mathsfit}{\encodingdefault}{\sfdefault}{m}{sl}
\SetMathAlphabet{\mathsfit}{bold}{\encodingdefault}{\sfdefault}{bx}{n}

\def\gB{{\mathcal{B}}}

\def\gD{{\mathcal{D}}}

\def\gV{{\mathcal{V}}}

\def\gX{{\mathcal{X}}}

\def\sR{{\mathbb{R}}}

\DeclareMathOperator*{\argmax}{arg\,max}
\DeclareMathOperator*{\argmin}{arg\,min}

\usepackage[makeroom]{cancel}

\usepackage[preprint]{icml2026}

\usepackage{amsmath}
\usepackage{amssymb}
\usepackage{mathtools}
\usepackage{amsthm}

\usepackage[capitalize,noabbrev]{cleveref}

\theoremstyle{plain}
\newtheorem{theorem}{Theorem}[section]

\theoremstyle{definition}

\theoremstyle{remark}
\newtheorem{remark}[theorem]{Remark}

\usepackage[textsize=tiny]{todonotes}

\icmltitlerunning{Active Flow Matching}

\begin{document}

\twocolumn[
\icmltitle{Active Flow Matching}



\icmlsetsymbol{equal}{*}

\begin{icmlauthorlist}
\icmlauthor{Yashvir S. Grewal}{ANU,data61}
\icmlauthor{Daniel M.\ Steinberg}{data61}
\icmlauthor{Thang D.~Bui}{ANU}
\icmlauthor{Cheng Soon Ong}{ANU,data61}
\icmlauthor{Edwin V.\ Bonilla}{data61}
\end{icmlauthorlist}

\icmlaffiliation{data61}{Data61, CSIRO, Australia}
\icmlaffiliation{ANU}{School of Computing, College of Systems and Society, Australian National University, Australia}

\icmlcorrespondingauthor{Yashvir Grewal}{Yashvir.Grewal@anu.edu.au}

\icmlkeywords{Machine Learning, ICML}

\vskip 0.3in
]



\printAffiliationsAndNotice{}  

\begin{abstract}
Discrete diffusion and flow matching models capture complex, non-additive and non-autoregressive structure in high-dimensional objective landscapes through parallel, iterative refinement. 
However, their implicit generative nature precludes direct integration with principled variational frameworks for online black-box optimisation, such as variational search distributions (VSD) and conditioning by adaptive sampling (CbAS). 
We introduce \emph{Active Flow Matching (AFM)}, which reformulates variational objectives to operate on conditional endpoint distributions along the flow, enabling gradient-based steering of flow models toward high-fitness regions while preserving the rigour of VSD and CbAS. We derive forward and reverse Kullback-Leibler (KL) variants using self-normalised importance sampling. Across a suite of online protein and small molecule design tasks, forward-KL AFM consistently performs competitively compared to state-of-the-art baselines, demonstrating effective exploration-exploitation under tight experimental budgets.
\end{abstract}

\section{Introduction}
%
%

Complex  non-additive interactions characteristic of high-dimensional objective landscapes such as those arising in protein design \citep{Starr2016EpistasisProtein,Phillips2008Epistasis} limit the effectiveness of standard autoregressive (AR) models. Indeed, these interactions, known as \textit{epistatic} in protein evolution, violate the underlying sequential factorisation of AR models, which commits to each token before observing its downstream context. 

As an alternative, discrete flow matching (DFM) generates sequences through parallel, iterative refinement \citep{Austin2021D3PM,Gat2024DFM}. At each step, token updates are conditioned on the full sequence context, allowing the model to coordinate long-range dependencies that sequential factorisation can fail to capture. Consequently, DFM models now match or exceed AR baselines across challenging protein and RNA design tasks \citep{Alamdari2023EvoDiff,Watson2023RFdiffusion,Ingraham2023Chroma}.

However, translating these generative capabilities into practical discoveries requires satisfying finite experimental budgets. 
Not only the design space is combinatorial (e.g., $20^{20}\!\approx\!10^{26}$ for 20-residue peptides) but also experimental validation remains expensive \citep{ScrippsSPRFees,DukeBIAFees}. To this end, active generation frameworks such as variational search distributions \citep[VSD,][]{Steinberg2025VSD} and conditioning by adaptive sampling \citep[CbAS,][]{Brookes2019CbAS} generate plausible sequences \textit{and} concentrate probability mass on rare, high-fitness regions, while maintaining diversity. Although methodologically principled, as they are underpinned by approximate probabilistic inference,  these approaches come with their own limitation: that of requiring a tractable density $q_\phi(\rvx)$, either to evaluate log-probabilities directly (CbAS) or to estimate score-function gradients (VSD).

Thus, it stands to reason that combining the flexibility of DFM models with active generation frameworks such as VSDs or CbAS provides a promising direction of research. However,   
state-of-the-art DFM and  discrete diffusion models are \emph{implicit} generators: 
they do not yield normalised distributions over discrete sequences. Consequently, evaluating or differentiating \(\log q_\phi(\rvx)\) is generally intractable. 
More specifically, for DFM, current formulations provide no simple closed-form mass function \citep{Lipman2022FlowMatching,Gat2024DFM}; and for discrete diffusion, exact \(\log q_\phi(\rvx)\) requires summing over exponentially many corruption paths \citep{Austin2021D3PM}. As a result, a na\"ive application of active generation methods that require \(\log q_\phi(\rvx)\) or its gradient \(\nabla_\phi \log q_\phi(\rvx)\) to these generative models is challenging.

\paragraph{Active Flow Matching (AFM).}
We address this challenge by reformulating variational objectives to operate on the conditional endpoint distributions that flow models \emph{do} provide, rather than on the intractable marginal. This allows us to steer discrete flow models toward high-fitness regions using principled exploration-exploitation trade-offs inherent to VSD and CbAS.
We derive both forward-KL and reverse-KL variants of AFM using self-normalised importance sampling, and introduce a mixture proposal that balances exploration, exploitation, and refinement. Across a suite of online protein and small molecule design tasks, forward-KL AFM performs competitively compared to existing baselines, and in many cases, discovers higher-scoring designs under tight experimental budgets.


\section{Background}
\subsection{Active Generation}

Consider a discrete design space $\gX$ (e.g., $\gV^L$ for sequences of length $L$ and vocabulary size $\gV$) with an unknown fitness function $f: \gX \to \sR$. We can only evaluate $f$ through expensive experiments that return noisy measurements:
\begin{align}
y = f(\rvx) + \epsilon, \quad \epsilon \sim p(\epsilon), \quad 
\mathbb{E}[\epsilon] = 0. \nonumber
\end{align}
Our objective is to discover designs in the \emph{super level-set} 
$\mathcal{S}_\tau = \{\rvx \in \mathcal{X} : y \geq \tau\}$ 
for a specified threshold $\tau$. Crucially, $\mathcal{S}_\tau$ is 
exponentially rare in high-dimensional spaces, making random 
sampling from any fixed prior $p(\rvx)$ ineffective. \citet{Steinberg2025VSD} formalise this as the \emph{active generation} problem: starting from initial observations $\gD_0 = \{(\rvx_n, y_n)\}_{n=1}^N$, iteratively proposing batches 
$\gB_r = \{\rvx_b\}_{b=1}^B$ for evaluation over $R$ rounds, and 
accumulating data $\gD_r = \gD_{r-1} \cup 
\{(\rvx_b, y_b)\}_{b \in \gB_r}$. The goal is to learn a 
generative model $q^\phi_{\mathcal{D}_r}(\rvx) \approx p(\rvx \mid y \geq \tau)$ that concentrates probability on $\mathcal{S}_\tau$ while maintaining diversity for batch evaluations. Active generation is also used for black-box optimisation, $\argmax_\rvx f(\rvx)$, by monotonically increasing the threshold such that $\tau_r > \tau_{r-1}$.

This poses three key challenges: (i) the generative model must adapt to sparse, sequential feedback rather than a fixed dataset, (ii) batch proposals require exploration-exploitation tradeoffs absent in supervised learning, and (iii) the discrete, combinatorial nature of $\mathcal{X}$ precludes \textit{direct} gradient-based search.
 VSD and CbAS both take a density estimation approach to solving this problem, with CbAS using the forward Kullback-Leibler (KL) divergence,
\begin{equation}
\phi^*_r = \argmin_\phi
\mathrm{KL}\!\left( 
 p(\rvx \mid y \geq \tau) \|
 q^\phi_{\mathcal{D}_r}(\rvx)
\right),
\end{equation}
while VSD uses the reverse,
\begin{equation}
\phi^*_r = \argmin_\phi
\mathrm{KL}\!\left( 
 q^\phi_{\mathcal{D}_r}(\rvx)\|
 p(\rvx \mid y \geq \tau)
\right).
\end{equation}
It is worth noting that CbAS was originally design for offline optimisation tasks, but \cite{Steinberg2025VSD} and others have adapted it to the online setting.

One key aspect of an efficient implementation of CbAS and VSD in \citet{Steinberg2025VSD,Steinberg2025AGPS}
is the observation that the conditioning set $\mathbb{I}\{ y \geq \tau \}$ can be approximated efficiently
by a binary classification task, i.e. by a class probability estimator $p(z=1 | x)$, where $z=1$
when $y \geq \tau$.

\subsection{Discrete Flow Matching}

Discrete flow matching \citep{Gat2024DFM,Campbell2024Generative} provides a principled framework for non-autoregressive generation of discrete sequences. The key insight is that while samples remain discrete at all times, probability mass flows continuously along a time-dependent path, enabling iterative refinement through parallel updates across all positions in the sequence. 
We next provide a concise summary of discrete flow matching; for more details, we refer the reader to \citet{lipman2024flowmatchingguidecode}.

\paragraph{Setup.} Consider discrete sequences $\rvx \in \gV^L$ where each of $L$ positions takes one of $|\gV|$ vocabulary values. The generative modelling task is to transform samples from a simple \emph{source} distribution $p(\rvx_0)$ into samples from the \textit{data} distribution $q(\rvx_1)$. Two natural choices for the source are: (i) \emph{uniform}: $p^{\text{unif}}(\rvx) = 1/|\mathcal{V}|$ for each position, and (ii) \emph{mask}: $p^{\text{mask}}(\rvx) = \delta_M(\rvx)$ concentrating mass on a special mask token $M$. 
For a single token position, we use $\delta_y(x)$ to denote the \emph{Kronecker delta}, which represents a point mass distribution that places all probability mass on state $y$: $\delta_y(x) = 1  \text{ if } x = y,  \text{ else } 0$.
For sequences, $\delta_\rvy(\rvx)$ is 1 if and only if all positions match: $\delta_\rvy(\rvx) = \prod_{i=1}^L \delta_{y^i}(x^i)$.

During training, we access pairs $(\rvx_0, \rvx_1)$ from a coupling $\pi(\rvx_0, \rvx_1)$ satisfying marginal constraints $\sum_{\rvx_1} \pi(\rvx_0, \rvx_1) = p(\rvx_0)$ and $\sum_{\rvx_0} \pi(\rvx_0, \rvx_1) = q(\rvx_1)$. The simplest choice is \emph{independent coupling}: $\pi(\rvx_0, \rvx_1) = p(\rvx_0) q(\rvx_1)$, though optimal transport couplings may improve training \citep{Pooladian2023Multisample}.

\paragraph{Probability paths.} We define a time-dependent \emph{probability path} $p_t(x)$ that smoothly interpolates between source and target over $t \in [0,1]$, satisfying $p_0 = p$ and $p_1 = q$. The path is constructed via a \emph{conditional probability path} $p_t(x \mid x_0, x_1)$ that interpolates between each training pair. For the mask source, a natural choice is \emph{convex interpolation}:
\begin{equation}
p_t(x \mid x_0, x_1) = (1 - \kappa_t) \delta_{x_0}(x) + \kappa_t \delta_{x_1}(x)
\label{eq:conditional_path}
\end{equation}
where $\kappa_t : [0,1] \to [0,1]$ is a monotonically increasing \emph{scheduler} satisfying $\kappa_0 = 0$ and $\kappa_1 = 1$. Common choices include linear ($\kappa_t = t$) or quadratic ($\kappa_t = t^2$) schedules.
Equation~\ref{eq:conditional_path} describes a stochastic process where at time $t$, we observe the source token $x_0$ with probability $(1-\kappa_t)$ and the target token $x_1$ with probability $\kappa_t$. As time progresses, the target is gradually revealed. Crucially, the state at time $t$ is always exactly one of $\{x_0, x_1\}$. The interpolation occurs in \emph{probability space}, not token space. The marginal path can be obtained by $p_t(x) = \sum_{x_0, x_1} p_t(x \mid x_0, x_1) \pi(x_0, x_1)$.

\paragraph{Continuous-time Markov chain dynamics.}
To sample from the target distribution $q$, we require a mechanism to evolve samples along the probability path from $t=0$ to $t=1$. In discrete state spaces, this evolution is governed by a continuous-time Markov chain \citep{Campbell2022Continuous} characterised by a time-dependent rate matrix. For an infinitesimal timestep $dt$, the probability of transitioning from state $x_t$ to state $x'$ is given by a rate $R_t(x' \mid x_t) dt$. 
In practice, trajectories are simulated with finite timesteps $h$:
\begin{equation}
p(x_{t+h} = x' \mid x_t) = \delta_{x_t}(x') + h \cdot u_t(x' \mid x_t) + o(h)
\end{equation}
where $u_t(x' \mid x_t)$ is the \emph{probability velocity} targeting state $x'$, satisfying $\sum_{x'} u_t(x' \mid x_t) = 0$ and $u_t(x' \mid x_t) \geq 0$ for $x' \neq x_t$ to ensure valid probability transitions.

\paragraph{Generating velocity from posteriors.}
A central result of \citet{Gat2024DFM} shows that the probability velocity generating the path \eqref{eq:conditional_path} is fully determined by the \emph{posterior distribution} $p_{1|t}(x_1 \mid x_t)$, which estimates the probability of the final target $x_1$ given the current state $x_t$:
\begin{equation}
u_t(x' \mid x_t) = \frac{\dot{\kappa}_t}{1 - \kappa_t} [p_{1|t}(x' \mid x_t) - \delta_{x_t}(x')]
\label{eq:velocity}
\end{equation}
where $\dot{\kappa}_t$ is the time derivative of the scheduler. The term $\delta_{x_t}(x')$ ensures rows sum to zero, satisfying the Kolmogorov forward equation. This velocity yields an elegant interpretation: at each timestep, probability mass flows \emph{toward} tokens $x'$ that are likely targets (according to $p_{1|t}$) and \emph{away} from the current token $x_t$, with a rate controlled by the signal-to-noise ratio $\dot{\kappa}_t/(1-\kappa_t)$. Early in generation (small $\kappa_t$), evolution is slow and exploratory; as $\kappa_t \to 1$, evolution accelerates to lock onto the target.


\paragraph{Learning the posterior.}
Since the true posterior $p_{1|t}(x_1 \mid x_t)$ is intractable, we train a parameterised neural network $q_\phi(x_1 \mid x_t, t)$ to approximate it. The standard training objective minimises the expected cross-entropy between the model prediction and the true target $x_1$:
\begin{equation}
\mathcal{L}_{\mathrm{CE}}(\phi) = \mathbb{E}_{t, x_0, x_1, x_t} [-\log q_\phi(x_1 \mid x_t, t)].
\label{eq:dfm_loss}
\end{equation}
This objective corresponds to a denoising task: (a) sample a time $t \sim \mathcal{U}[0,1]$ and a data pair $(x_0, x_1) \sim \pi$; (b) sample an intermediate state $x_t$ from the conditional path $p_t(x \mid x_0,x_1)$ in \eqref{eq:conditional_path}; and (c) update $\phi$ to maximise the likelihood of the clean target $x_1$ given the noisy state $x_t$.


\paragraph{Sampling algorithm.}
At inference, we initialise $x_0 \sim p_0$ and iteratively evolve the sample using the Euler-discretised dynamics. For timesteps $t = 0, h, 2h, \ldots, Kh$, where $h = 1/K$, the state $x_t$ is updated to $x_{t+h}$ by sampling from the categorical transition kernel, $x_{t+h} \sim \text{Cat}\left( \delta_{x_t}(\cdot) + h \cdot u_t(\cdot \mid x_t) \right)$.
For a mask source $p_0(x) = \delta_M(x)$, the process begins with a fully masked sequence.  As $t$ advances, the velocity field $u_t$ drives probability mass away from the mask token $M$ and toward specific token predicted by the posterior, effectively `unmasking' the sequence. The number of steps $K$ trades off sample quality against computational cost \citep{Gat2024DFM}.

\paragraph{Extension to sequences.}
For sequences of length $L$, the generative process factorises over positions, but couples through the velocity. Specifically, the conditional probability path factorises independently:
\begin{equation}
p_t(\rvx \mid \rvx_0, \rvx_1) = \prod_{i=1}^L p_t(x^i \mid x_0^i, x_1^i).
\end{equation}
However, the learned velocity couples these positions. The approximate posterior $q_\phi(x_1^i \mid \rvx_t, t)$ predicts the target token at position $i$ by conditioning on the \emph{entire} current sequence $\rvx_t$. This allows the model to capture complex, non-local dependencies while maintaining $O(1)$ parallel generation steps. This contrasts with autoregressive models, which require $O(L)$ sequential steps.

\paragraph{The implicit generator problem.}
While discrete flow matching excels at capturing global structure, it is fundamentally an \emph{implicit} generator. The model optimises the cross-entropy objective \eqref{eq:dfm_loss} to learn the conditional posterior $q_\phi(\rvx_1 \mid \rvx_t, t)$, which defines the vector field. However, this formulation provides \emph{no closed-form expression for the marginal likelihood $q_\phi(\rvx)$}. Directly computing the likelihood $\log q_\phi(\rvx)$ requires marginalizing over all possible stochastic trajectories from the source to $\rvx$:
\begin{equation}
q_\phi(\rvx) = \sum_{\rvx_0} p(\rvx_0) \sum_{\text{paths } \gamma: \rvx_0 \to \rvx} \mathbb{Q}_\phi[\gamma].
\end{equation}
For discrete sequences, the number of paths is combinatorial in $L$ and the number of steps $K$, rendering this summation intractable. Consequently, the gradient $\nabla_\phi \log q_\phi(\rvx)$ is also unavailable.

This intractability creates a conflict with principled active generation: VSD requires $\nabla_\phi \log q_\phi(\rvx)$ to estimate the evidence lower bound (ELBO) gradient, while CbAS requires explicit density ratios involving $q_\phi(\rvx)$. Standard active generation frameworks are thus mathematically incompatible with implicit flow models. We next resolve this by reformulating these objectives to operate on the \emph{conditional endpoint distributions}, which are naturally provided by the flow.


\section{Active Flow Matching}

Having established that discrete flow matching models provide conditional endpoint distributions but lack a tractable marginal likelihood $\log q_\phi(\rvx)$, we reformulate variational objectives to operate directly on these conditional distributions along the flow path.

\subsection{Problem Setup}
We consider an active generation setting over $R$ rounds. At the $r$-th round, we assume to have collected a labelled dataset $\mathcal{D}_r = \{(\rvx_n, y_n)\}_{n=1}^{N_r}$. Our goal is to learn an \textit{active flow} matching model with parameters $\phi$ that concentrates probability mass on the superlevel set $\mathcal{S}_\tau = \{\rvx \in \mathcal{X} : y \geq \tau\}$, while maintaining diversity for the next batch proposal $\mathcal{B}_r$. This flow is parameterised by a denoising neural network that captures the model's conditional endpoint distribution $q_t^\phi(\rvx_1 \mid \rvx_t)$, where we drop the conditioning on $t$ for brevity.

In addition, we assume to have access to a \textit{base flow} with parameter $\theta$. This model utilises the same architecture as the active flow above and is parameterised by the reference distribution $q_t^\theta(\rvx_1 \mid \rvx_t)$. It is typically derived from a prior optimisation state, such as the solution from the $(r-1)^{\text{th}}$ round or a lagging snapshot of the current active flow model. This will be used for the reverse-KL objective and the proposal mixture for importance sampling. Lastly, the dataset $\mathcal{D}_r$ is used to train a \textit{classifier} that estimates the probability of high fitness, $p_{\mathcal{D}_r}(y \geq \tau \mid \rvx)$, for a pre-defined value $\tau$, which is used for conditioning the generative model.

\paragraph{Key Insight.}
Standard variational methods (VSD, CbAS) minimise divergences involving the marginal $q_\phi(\rvx)$. Since computing this marginal requires intractable summation over all discrete paths, we lift the objective to the \emph{conditional flow path}. Instead of matching the final distribution at $t=1$, we match the \emph{conditional distributions} averaged over flow time $t \in [0,1]$ and intermediate states $\rvx_t$. This leverages the flow's tractable denoiser $q_t^\phi(\rvx_1 \mid \rvx_t)$, while still providing samples from the desired marginal distribution in the case of forward-KL. We now describe the proposed approaches and the implementation details.


\subsection{Forward-KL Active Flow Matching}

Following CbAS \citep{Brookes2019CbAS}, we minimise the forward-KL between the true conditional endpoint distribution and that of the active flow matching model:
\begin{align}
&\mathcal{L}_{\mathrm{fwd}}(\phi) \nonumber \\
&~= \mathbb{E}_{t, \rvx_t \mid y \geq \tau} 
\left[ 
\mathrm{KL}\!\left( 
p_t(\rvx_1 \mid \rvx_t, y \geq \tau) 
\,\|\, 
q_t^\phi(\rvx_1 \mid \rvx_t) 
\right) 
\right],
\nonumber\\
&~= -\mathbb{E}_{t, \rvx_1 \mid y \geq \tau, \rvx_t \mid \rvx_1}
\left[ \log q_t^\phi(\rvx_1 \mid \rvx_t) \right]
+ C.
\end{align}
Here $C$ is constant w.r.t.\ $\phi$, and where the target posterior is defined using the base flow $\theta$:
\begin{align}
p_t(\rvx_1 \mid \rvx_t, y \geq \tau) 
\propto q_t^\theta(\rvx_1 \mid \rvx_t) \, p_{\mathcal{D}_r}(y \geq \tau \mid \rvx_1).
\nonumber
\end{align}
%
Since we cannot sample from $p(\rvx_1 \mid y \geq \tau)$, we leverage self-normalised importance sampling (SNIS) with a proposal distribution $\mu(\rvx_1)$:
\begin{equation}
\mathcal{L}_{\mathrm{fwd}}(\phi) 
\approx -\mathbb{E}_{t}\!
\left[ 
\frac{
\sum_{k=1}^K w_k \log q_t^\phi(\rvx_{1,k} \mid \rvx_{t,k})
}{
\sum_{k=1}^K w_k
} 
\right],
\label{eq:fwd_snis}
\end{equation}
where $\{\rvx_{1,k}\}_{k=1}^K \sim \mu(\rvx_1)$, weights $w_k = p_{\mathcal{D}_r}(y \geq \tau \mid \rvx_{1,k}) / \mu(\rvx_{1,k})$, $t \sim \mathrm{Unif}[0,1]$, and $\rvx_{t,k} \sim p_t(\rvx_t \mid \rvx_{1,k})$.

We will now show that optimising the above objective will yield samples from the desired marginal distribution.

\begin{theorem}[Consistency of Forward-KL AFM]
\label{thm:fwd_consistency}
Let $p^*(\rvx) \propto p_{\mathrm{prior}}(\rvx) w(\rvx)$ be the target distribution, where $w(\rvx) = p(y \ge \tau \mid \rvx)$ and $p_{\mathrm{prior}}$ is the uniform distribution. Under standard DFM assumptions with masked-source coupling, convex interpolant paths (Eq.~\ref{eq:conditional_path}), and a strictly positive scheduler, the global minimiser $\phi^*$ of the Forward-KL AFM objective (Eq.~\ref{eq:fwd_snis}) yields a terminal distribution $q_1^{\phi^*} = p^*$ almost everywhere.
\end{theorem}

\begin{proof}[Proof Sketch]
We reduce the importance-weighted training to standard DFM on $p^*$. The key identity is:
\begin{equation}
\mathbb{E}_{\rvx \sim p_{\mathrm{prior}}}\bigl[w(\rvx) \mathcal{L}_{\mathrm{CE}}(\rvx; \phi)\bigr] = Z \cdot \mathbb{E}_{\rvx \sim p^*}[\mathcal{L}_{\mathrm{CE}}(\rvx; \phi)]
\end{equation}
where $Z$ is the normalisation and $\mathcal{L}_{\mathrm{CE}}$ is in Eq.\ \ref{eq:dfm_loss}. Since $Z$ is independent of $\phi$, minimising the weighted objective over the prior is equivalent to minimising the standard DFM loss over the target. Consistency follows from \citet{Gat2024DFM}: the minimiser recovers the true posterior under $p^*$, which defines a velocity field that generates the correct marginal path. Full proof in Appendix~\ref{app:proofs}.
\end{proof}

\begin{remark}
\label{rem:importance_sampling}
This result highlights that AFM requires no new flow dynamics theory. The weights implicitly reweight the training data from the prior to $p^*$. In practice (Algorithm~\ref{alg:afm_fwd}), we estimate this objective using samples from a mixture proposal $\mu(\rvx)$ (Eq.\ \ref{eq:mixture}) rather than just the prior. By correcting these samples with importance weights, we obtain consistent gradient estimates for the target objective without ever requiring samples from $p^*$ itself.
\end{remark}


\subsection{Reverse-KL Active Flow Matching}

Following VSD \citep{Steinberg2025VSD}, we minimise:
\begin{align}
&\mathcal{L}_{\mathrm{rev}}(\phi) \nonumber \\
&~= \mathbb{E}_{t, \rvx_t \mid y \geq \tau} 
\left[ 
\mathrm{KL}\!\left( 
q_t^\phi(\rvx_1 \mid \rvx_t) 
\,\|\, 
p_t(\rvx_1 \mid \rvx_t, y \geq \tau) 
\right) 
\right].
\nonumber\\ 
&~= \mathbb{E}_{t, \rvx_t \mid y \geq \tau, \rvx_1 \sim q_t^\phi(\cdot \mid \rvx_t)} 
\Big[ g(\rvx_1, \rvx_t)
\Big],
\end{align}
where $g(\rvx_1, \rvx_t) = \log q_t^\phi(\rvx_1 \mid \rvx_t) - \log q_t^\theta(\rvx_1 \mid \rvx_t) 
- \log p_{\mathcal{D}_r}(y \geq \tau \mid \rvx_1)$.
To approximate the expectation involving $\rvx_t \mid y \geq \tau$, we can use SNIS as above to yield,
\begin{align}
\mathcal{L}_{\mathrm{rev}}(\phi) 
&\approx \mathbb{E}_{t}\! \left[
\sum_{k=1}^K \widetilde{w}_k \, 
\mathbb{E}_{\rvx_1 \sim q_t^\phi(\cdot \mid \rvx_{t,k})} 
\big[ 
g(\rvx_1, \rvx_{t,k})
\big]\right],
\label{eq:rev_snis}
\end{align}
where $\{\rvx_{1,k}\}_{k=1}^K \sim \mu(\rvx_1)$, $w_k = p_{\mathcal{D}_r}(y \geq \tau \mid \rvx_{1,k}) / \mu(\rvx_{1,k})$, and $\widetilde{w}_k = w_k / \sum_{j} w_j$.
Unfortunately, unlike forward-KL AFM, we have not yet managed to prove that the reverse-KL AFM yields a consistent estimate of the desired marginal distribution.

\begin{algorithm}[t]
\caption{$r$-th round of Active Flow Matching (AFM)\label{algo:afm}}
\label{alg:afm}
\label{alg:afm_fwd}
\label{alg:afm_rev}
\begin{algorithmic}[1]
\STATE \textbf{Input:} Data $\mathcal{D}_r$, threshold $\tau$, proposal $\mu$, batch size $K$
\STATE \textbf{Input:} Active flow $q_\phi$, classifier $p_\mathcal{D}(y \geq \tau \mid \rvx)$, base flow $p_\theta$, round $r$
\STATE \textbf{Input:} Objective $\in \{\textsc{Fwd-KL}, \textsc{Rev-KL}, \textsc{Sym-KL}\}$
\FOR{each gradient step}
    \STATE Sample $\{\rvx_{1,k}\}_{k=1}^K \sim \mu(\rvx_1)$
    \STATE Compute $w_k = p_\mathcal{D}(y \geq \tau \mid \rvx_{1,k}) / \mu(\rvx_{1,k})$
    \STATE Normalise $\widetilde{w}_k = w_k / \sum_j w_j$
    \STATE Sample $t \sim \mathrm{Unif}[0,1]$
    \STATE Sample $\rvx_{t,k} \sim p_t(\cdot \mid \rvx_{1,k})$ for all $k$ (see Eq.~\ref{eq:conditional_path})
    \IF{\textsc{Fwd-KL}}
        \STATE $\mathcal{L} \gets$ Eq.~\ref{eq:fwd_snis}
    \ELSE
        \STATE Sample $\rvx_{1,k}' \sim q_\phi(\cdot \mid \rvx_{t,k}, t)$ for all $k$
        \STATE $\mathcal{L} \gets$ Eq.~\ref{eq:rev_snis} \textbf{if} \textsc{Rev-KL} \textbf{else} Eq.~\ref{eq:sym_kl_obj}
    \ENDIF
    \STATE Update $\phi \gets \phi - \eta \nabla_\phi \mathcal{L}$
\ENDFOR
\STATE \textbf{Output:} Updated flow $q_\phi$
\end{algorithmic}
\end{algorithm}

\subsection{Symmetric-KL Active Flow Matching}

We combine both objectives by minimising the symmetric KL divergence:
\begin{equation}
\mathcal{L}_{\mathrm{sym}}(\phi) = \mathcal{L}_{\mathrm{fwd}}(\phi) + \mathcal{L}_{\mathrm{rev}}(\phi),
\label{eq:sym_kl_obj}
\end{equation}
where $\mathcal{L}_{\mathrm{fwd}}$ and $\mathcal{L}_{\mathrm{rev}}$ are defined in Eqs.~\ref{eq:fwd_snis} and~\ref{eq:rev_snis} respectively. The symmetric KL inherits mode-covering behaviour from the forward term and mode-seeking behaviour from the reverse term. The implementation of the proposed approaches is summarised in \cref{alg:afm}.



\subsection{Proposal Distribution Design}

A critical component of both forward- and reverse-KL AFM is the choice of proposal distribution $\mu(\rvx_1)$ for importance sampling. A good proposal must balance three competing objectives: (i) broad coverage over the sequence space to avoid missing promising regions, (ii) concentration on high-fitness designs to reduce estimator variance, and (iii) computational tractability. We introduce a three-component mixture that addresses each objective.

\paragraph{Mixture definition.}
At active generation round $r$, we define the proposal as:
\begin{equation}
\mu(\rvx_1) = \alpha_0 \, p_0(\rvx_1) 
       + \alpha_{\mathrm{flow}} \, q_1^\theta(\rvx_1) 
       + \alpha_{\mathrm{rbuff}} \sum_{j=1}^{J} \pi_j \, \delta_{x^{(j)}}(\rvx_1),
\label{eq:mixture}
\end{equation}
where $\alpha_0$, $\alpha_{\mathrm{flow}}$, and $\alpha_{\mathrm{rbuff}}$ are non-negative mixing coefficients subject to the constraint $\alpha_0 + \alpha_{\mathrm{flow}} + \alpha_{\mathrm{rbuff}} = 1$, $p_0$ is the uniform prior over sequences, $q_1^\theta$ is the marginal endpoint distribution of the base flow $\theta$ (from the previous round), and $\{\rvx^{(j)}\}_{j=1}^J$ is a replay buffer of high-scoring sequences with normalised weights:
\begin{equation}
\pi_j = \frac{\exp(\gamma \, y^{(j)})}{\sum_{j'=1}^J \exp(\gamma \, y^{(j')})},
\label{eq:buffer_weights}
\end{equation}
where $y^{(j)}$ denotes the observed fitness score and $\gamma > 0$ is a temperature parameter controlling concentration on top sequences.

\paragraph{Prior component: weight simplification.}
When the proposal is the prior, $\mu(\rvx_1) = p_0(\rvx_1)$, the importance weights simplify considerably. The unnormalised posterior is:
\begin{equation}
\tilde{p}(\rvx_1 \mid y \geq \tau) = p(y \geq \tau \mid \rvx_1) \cdot p_0(\rvx_1),
\end{equation}
where $p(y \geq \tau \mid \rvx_1)$ is the classifier likelihood and $p_0(\rvx_1)$ is the prior. For self-normalised importance sampling, we require unnormalised weights:
\begin{align}
w(x_1) &= \frac{\tilde{p}(\rvx_1 \mid y \geq \tau)}{\mu(\rvx_1)} 
= p(y \geq \tau \mid \rvx_1).
\end{align}
That is, the prior \emph{cancels exactly} in the importance weight ratio, regardless of whether $p_0(\rvx_1)$ is uniform or not, and the importance weight reduces to the classifier output. This eliminates the need to evaluate the prior density $p_0(\rvx_1)$, and is similar to the strategy in \cite{brookes2018design}.

\paragraph{Flow component: Monte Carlo integration.}
Sampling from the previous update's flow concentrates proposals near regions the model has already identified as promising, enabling \emph{local refinement}. However, unlike the prior and replay buffer, the flow's marginal density $q_1^\theta(\rvx_1)$ is not available in closed form. We therefore estimate it via Monte Carlo integration.

Recall that the base flow $\theta$ defines a stochastic process from source $p_0$ to data via the learned conditional $p_\theta(\rvx_1 \mid \rvx_0, t=1)$. The marginal density at the endpoint is:
\begin{equation}
q_1^\theta(\rvx_1) = \int p_0(\rvx_0) \, p_\theta(\rvx_1 \mid \rvx_0, t=1) \, d\rvx_0,
\label{eq:flow_marginal}
\end{equation}
where $\theta$ denotes the flow parameters from the previous update. Since the discrete state space makes this integral intractable, we approximate it using $N$ Monte Carlo samples $\{\rvx_0^{(n)}\}_{n=1}^N \sim p_0$:
\begin{equation}
q_1^\theta(\rvx_1) \approx \frac{1}{n} \sum_{n=1}^N p_\theta(\rvx_1 \mid \rvx_0^{(n)}, t=1).
\label{eq:flow_mc}
\end{equation}
This Monte Carlo estimator introduces variance, but in practice, we find this approximation sufficiently accurate for effective training.

\paragraph{Replay buffer component: unity weights from empirical observations.}
The replay buffer $\{\rvx_1^{(j)}\}_{j=1}^J$ stores sequences from previous rounds where we have \emph{directly observed} $f(\rvx_1^{(j)}) > \tau$, thus confirming they satisfy the constraint. When sampling a batch entirely from the replay buffer with proposal:
%
$\mu(\rvx_1) = \sum_{j=1}^J \pi_j \delta_{\rvx_1^{(j)}}(\rvx_1)$,
%
the unnormalised importance weight becomes:
\begin{align}
w(\rvx_1^{(j)}) &= \frac{p(y \geq \tau \mid \rvx_1^{(j)}) \cdot p_0(\rvx_1^{(j)})}{\pi_j} 
= \frac{p_0(\rvx_1^{(j)})}{\pi_j}.
\end{align}
If all buffer sequences have uniform prior probability, or if we set $\pi_j \propto p_0(\rvx_1^{(j)})$, then all weights are approximately equal. After SNIS normalisation, this yields $\widetilde{w}_k = 1/K$, giving \emph{zero variance from reweighting}. In practice, we use fitness-weighted sampling $\pi_j \propto \exp(\gamma y^{(j)})$ to concentrate on the highest-scoring observed sequences while maintaining low-variance gradient estimates.

\paragraph{Practical implementation.}
While Equation~\ref{eq:mixture} defines our proposal density, evaluating the full mixture $\mu(\rvx_1)$ for every sample requires computing all component densities, which is expensive. Instead, at each iteration we select a single component $c \sim \text{Categorical}(\alpha_0, \alpha_{\mathrm{flow}}, \alpha_{\mathrm{rbuff}})$ and draw the full batch from that component. Importance weights are then computed with respect to the selected component density only: $w(\rvx_1) = p_{\mathcal{D}_r}(y \geq \tau \mid \rvx_1)$ for the prior, $w(\rvx_1) = p_{\mathcal{D}_r}(y \geq \tau \mid \rvx_1) / q_1^\theta(\rvx_1)$ for the flow, and $w(\rvx_1) = 1$ for replay. This per-component scheme remains consistent while significantly reducing computational overhead.

\begin{figure*}[t]
    \centering
    \begin{subfigure}[t]{0.32\textwidth}
        \centering
        \includegraphics[width=\linewidth]{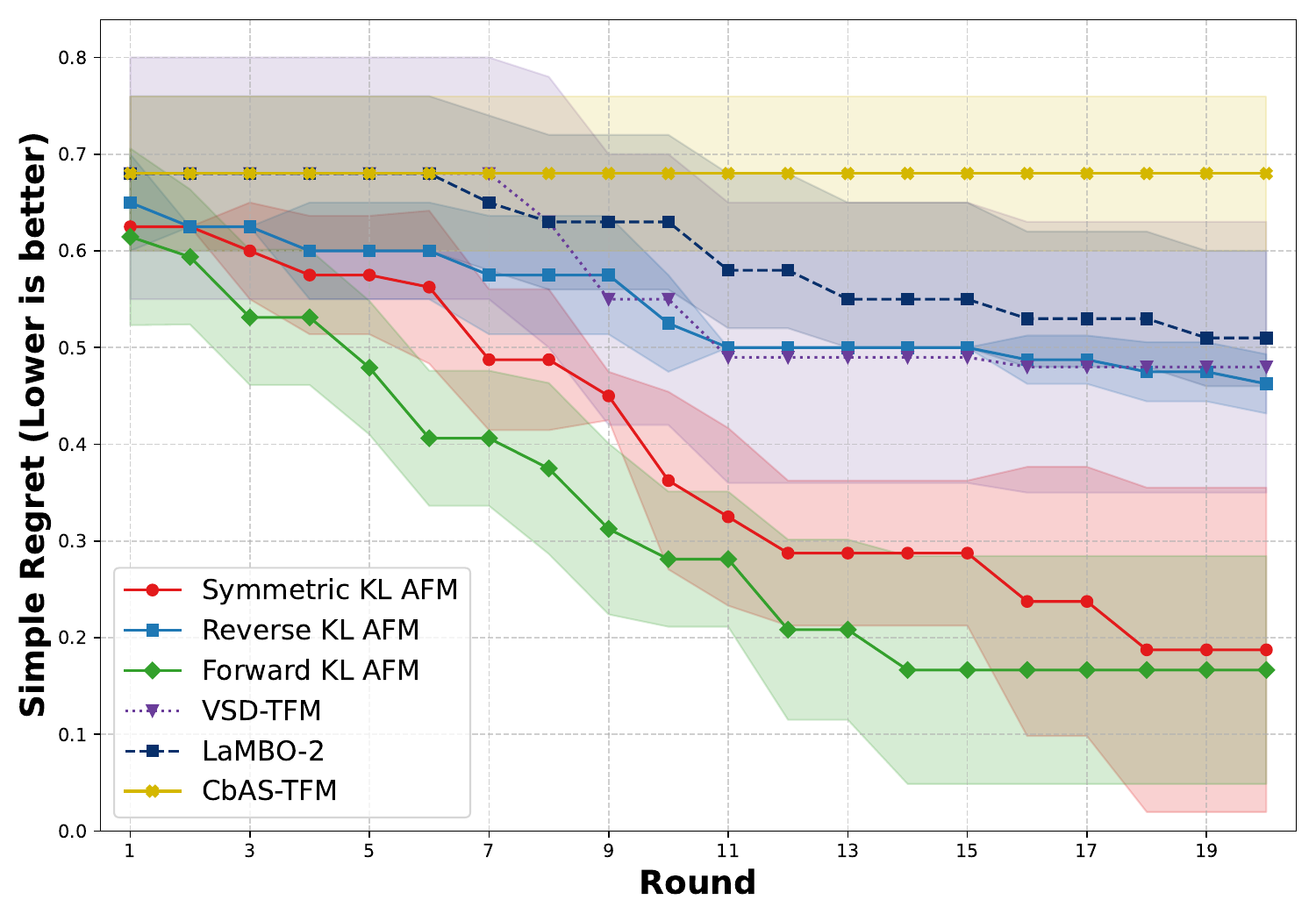}
        \caption{Ehrlich landscape, $L = 32$.}
        \label{fig:ehrlich_L32}
    \end{subfigure}
    \hfill
    \begin{subfigure}[t]{0.32\textwidth}
        \centering
        \includegraphics[width=\linewidth]{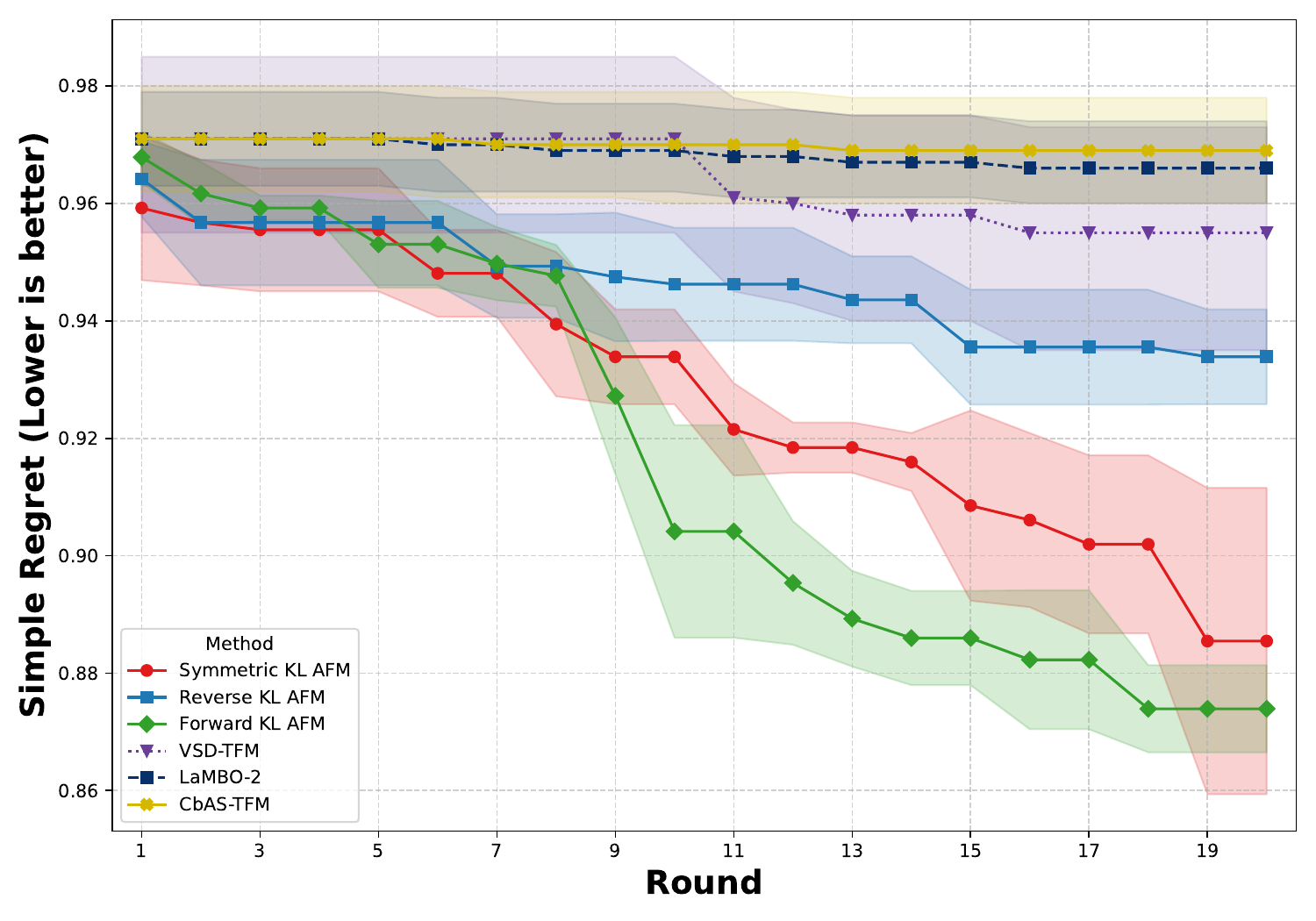}
        \caption{Ehrlich landscape, $L = 64$.}
        \label{fig:ehrlich_L64}
    \end{subfigure}
    \hfill
    \begin{subfigure}[t]{0.32\textwidth}
        \centering
        \includegraphics[width=\linewidth]{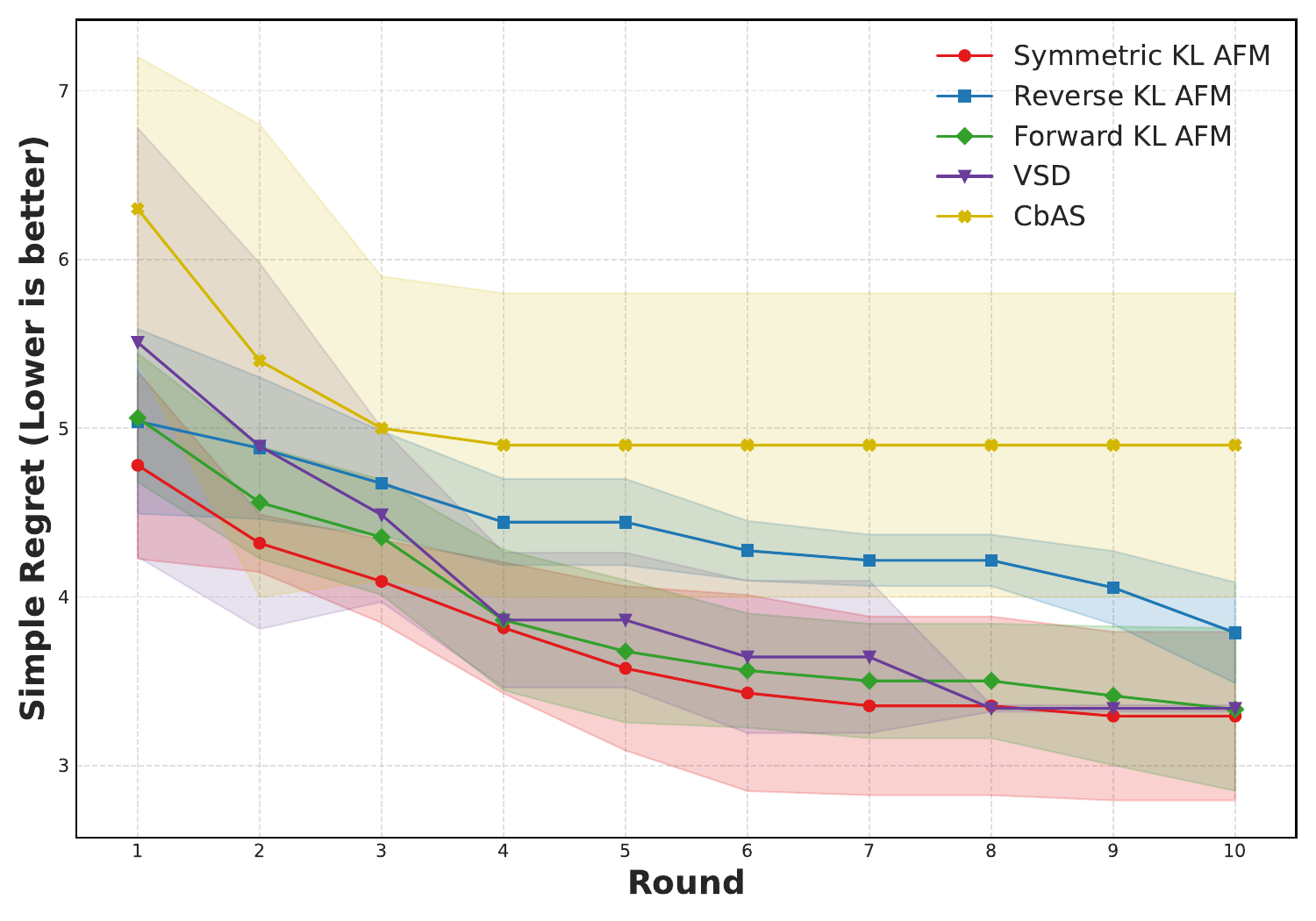}
        \caption{AAV landscape ($L=28$).}
        \label{fig:aav}
    \end{subfigure}
    \caption{
        Budget-constrained optimisation on Ehrlich synthetic landscapes and the AAV capsid design task.
        We report the simple regret
    }
    \label{fig:landscape_experiments}
\end{figure*}

\section{Related Work}

Autoregressive models such as ProGen2~\citep{Nijkamp2023ProGen2} and Tranception~\citep{Notin2022Tranception} achieve strong fitness prediction but their left-to-right factorisation limits epistatic coupling and flexible conditioning. Discrete diffusion addresses this via parallel denoising: D3PMs~\citep{Austin2021D3PM} generalise diffusion to finite state spaces, MDLM~\citep{Sahoo2024MDLM} connects absorbing-state diffusion to masked language modelling, and EvoDiff~\citep{Alamdari2023EvoDiff} enables structure-free motif-conditioned protein generation. Discrete flow matching~\citep{Gat2024DFM,Campbell2024Generative} learns continuous probability flows over discrete spaces via conditional endpoint distributions $q_\phi(\rvx_1 \mid \rvx_t, t)$. Crucially, all parallel generative models are \emph{implicit}: exact marginal evaluation requires summing over exponentially many paths~\citep{Austin2021D3PM}, motivating our reformulation in terms of tractable conditionals.

Active generation methods trains a surrogate $f_\psi(\rvx)$ on measurements and searches for high-scoring designs~\citep{Yang2019MLDE}. Two variational frameworks target $p(\rvx \mid y \geq \tau)$ directly: CbAS~\citep{Brookes2019CbAS} minimises forward-KL via importance-weighted retraining; VSD~\citep{Steinberg2025VSD} minimises reverse-KL using score-function gradients. Both require tractable $q_\phi(\rvx)$---for density ratios (CbAS) or score gradients (VSD)---which implicit flow models lack. We extend these principles by reformulating objectives over $q_t^\phi(\rvx_1 \mid \rvx_t)$. LaMBO-2 \cite{Gruver2023LaMBO2} instead uses expected improvement (EI) or expected hypervolume improvement (EHVI) acquisition functions as a guide for a diffusion generative backbone in an active generation setting.

Classifier~\citep{Dhariwal2021Classifier} and classifier-free~\citep{Ho2022ClassifierFree} guidance steer continuous diffusion via score perturbation. Discrete extensions require relaxations: Gumbel-Softmax~\citep{Jang2017GumbelSoftmax} enables differentiable sampling with biased gradients; straight-through estimators~\citep{Bengio2013StraightThrough} pass gradients through discrete operations; CTMC-based guidance~\citep{Nisonoff2024DiscreteGuidance} exploits single-token transitions. AFM differs fundamentally: we incorporate fitness into the \emph{training objective} over conditional endpoints, yielding exact discrete samples without relaxations or sampling modifications.



\section{Experiments}
\label{sec:experiments}

We evaluate AFM on five protein design and one small molecule task spanning a range of sequence lengths, fitness landscapes, and oracle characteristics. Our experiments are designed to answer two questions: (i) Can implicit generative models be steered towards high-fitness regions without access to tractable marginal likelihoods?
(ii) How do the forward-KL, reverse-KL and symmetric KL variants of AFM compare in performance?

\subsection{Tasks and Metrics}

\paragraph{Ehrlich synthetic landscapes \citep{Stanton2024Ehrlich}.}
Ehrlich functions are closed-form, procedurally generated test functions over discrete sequences that mimic key geometric properties of biophysical sequence optimisation problems, including ruggedness, constraints, and epistasis. Each instance defines a score $f(\rvx) \in [0, 1]$ for valid sequences by aggregating the satisfaction of motif-like constraints; invalid sequences receive $f(\rvx) = -1$. Each instance admits a known global optimum $\rvx^\star$ with $f(\rvx^\star) = 1.0$. We instantiate Ehrlich functions with fixed lengths $L \in \{32, 64\}$. We pre-train on a pool of $5{,}000$ unlabelled valid sequences to learn the feasible region of sequence space. We perform optimisation for $R=20$ rounds with a batch size of $B=128$. These tasks allow us to compare methods in a controlled setting where the global optimum is known.

\paragraph{Adeno-Associated Virus (AAV) \citep{Bryant2021AAV}.}
We also evaluate performance on the AAV capsid design task, following the black-box optimisation setting of \citet{Steinberg2025VSD}. The goal is to maximise the fitness of a 28-amino acid sequence ($20^{28}$ search space) corresponding to a region of the AAV capsid protein involved in viral packaging. We use a convolutional neural network (CNN) oracle provided by \citet{Kirjner2024GGS} as the ground-truth fitness function. The known global maximum of $y^* \approx 19.54$ allows for regret measurement. We perform optimisation for $R=10$ rounds with a batch size of $B=128$.

\paragraph{Structure-based protein design.}
We consider a structure-based protein optimisation task with two objectives: (i) improving thermodynamic stability ($-\Delta G$) and (ii) increasing solvent-accessible surface area (SASA). The black-box oracle is the FoldX molecular simulation software \citep{Schymkowitz2005FoldXWeb}, wrapped by the \textsc{poli} library \citep{gonzalez2024poli}. We use the mRouge red fluorescent protein ($L = 228$) as the base protein. We perform optimisation for $R=10$ rounds with a batch size of $B=8$ and a constraint of up to three mutations per round. 

 \paragraph{Molecular Docking (F2/Thrombin).}
We evaluate on molecular docking targeting Thrombin (F2) using the Dockstring  benchmark~\citep{garciaortegon2022dockstring}, with molecules represented as SELFIES~\citep{krenn2020selfies}. We optimise $f(\rvx) = s_{\text{dock}}(\rvx) - 7.5(1 - \text{QED}(\rvx))$, combining docking score with a drug-likeness penalty~\citep{bickerton2012quantifying}. Starting from ${\sim}260$k molecules with pre-computed scores, each method proposes $B=32$ molecules per round for $R=100$ rounds (3,200 oracle calls). We compare against VSD; We failed to stabilise CbAS on this task.

\paragraph{Evaluation metrics.}
All results are averaged over five independent runs (three for Dockstring). For tasks with known global optimum (Ehrlich, AAV), we report \emph{simple regret}: $\text{reg}_r = f(\rvx^\star) - \max_{s \leq r} \max_{\rvx \in \mathcal{B}_s} f(\rvx)$ (lower is better). For tasks without unknown optima (FoldX, Dockstring), we report the \emph{best score found}: $I_r = \max_{s \leq r} \max_{x \in \mathcal{B}_s} f(\rvx)$ (higher is better).

\begin{figure*}[t]
    \centering
    \begin{subfigure}[t]{0.32\textwidth}
        \centering
        \includegraphics[width=\linewidth]{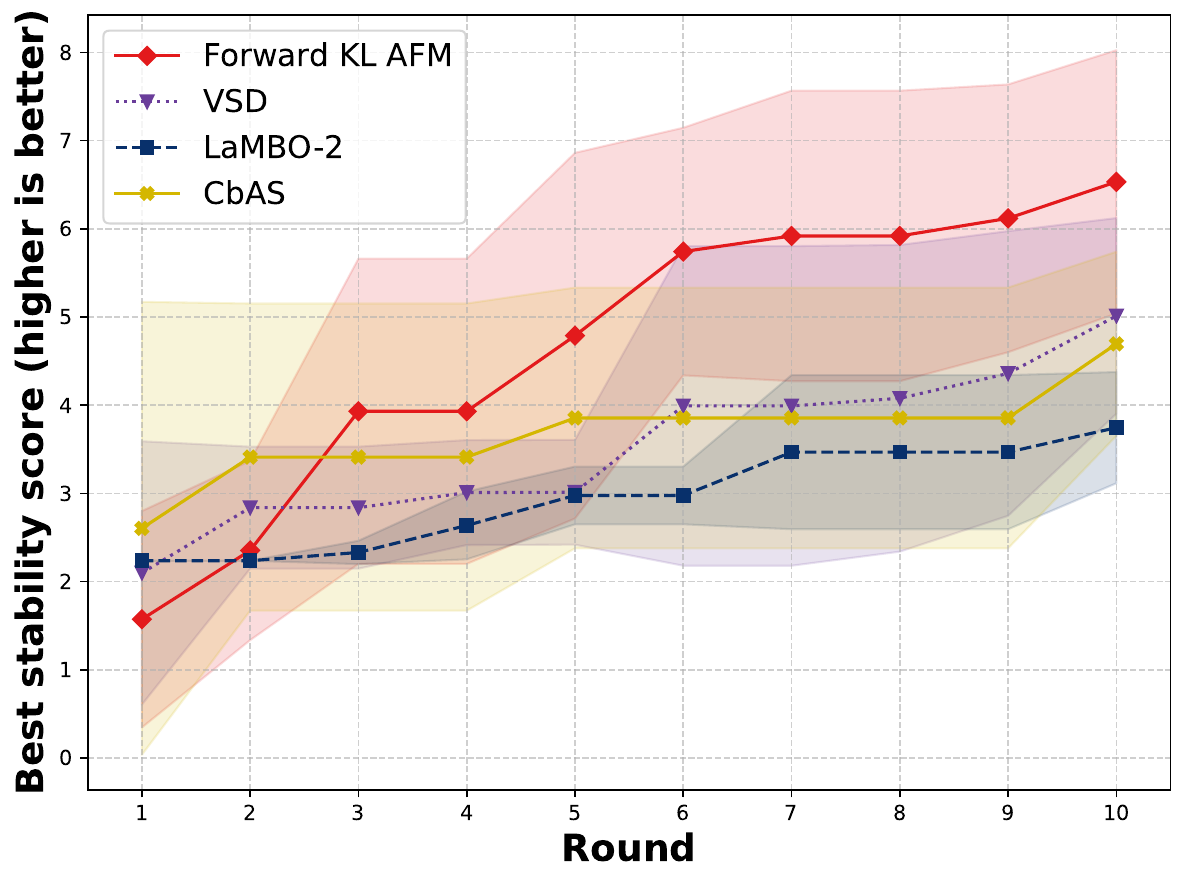}
        \caption{FoldX stability.}
        \label{fig:foldx_stability}
    \end{subfigure}
    \hfill
    \begin{subfigure}[t]{0.32\textwidth}
        \centering
        \includegraphics[width=\linewidth]{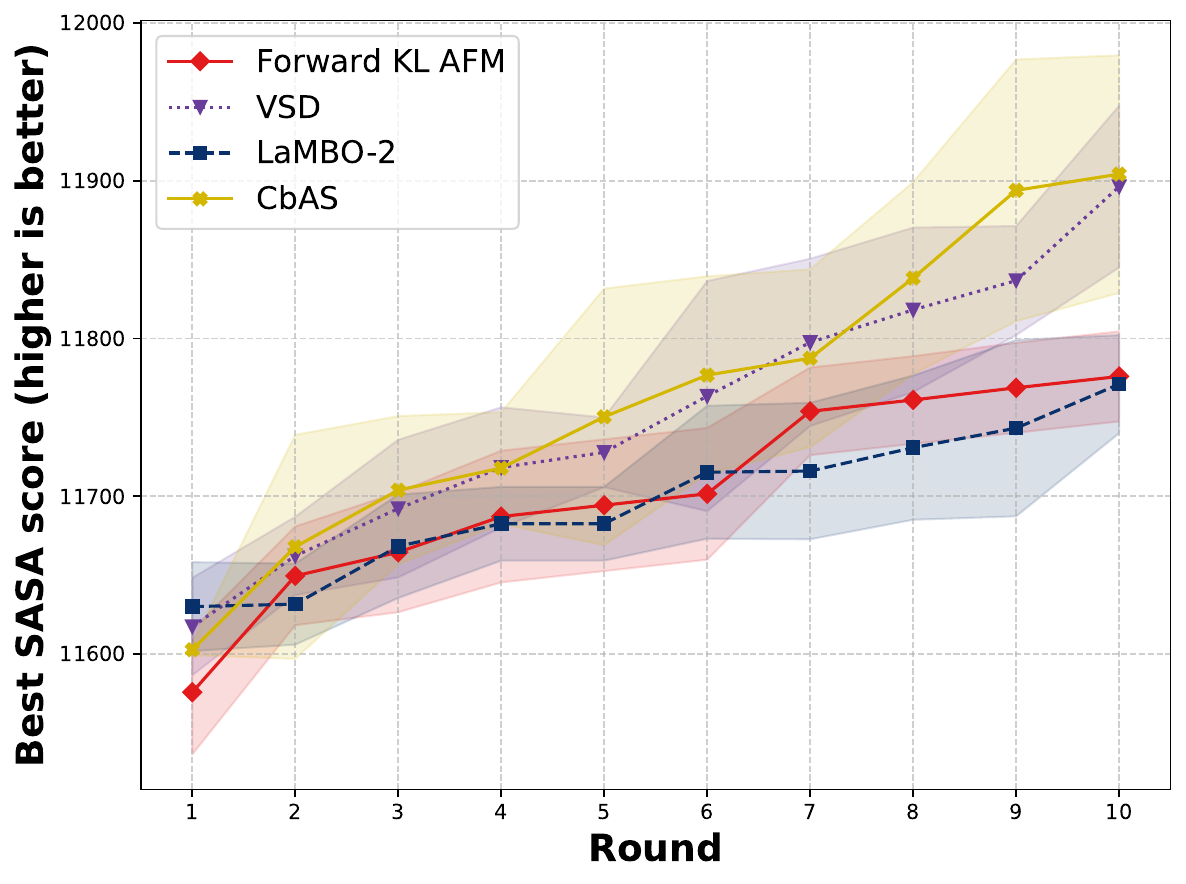}
        \caption{FoldX SASA.}
        \label{fig:foldx_sasa}
    \end{subfigure}
    \hfill
    \begin{subfigure}[t]{0.32\textwidth}
        \centering
        \includegraphics[width=\linewidth]{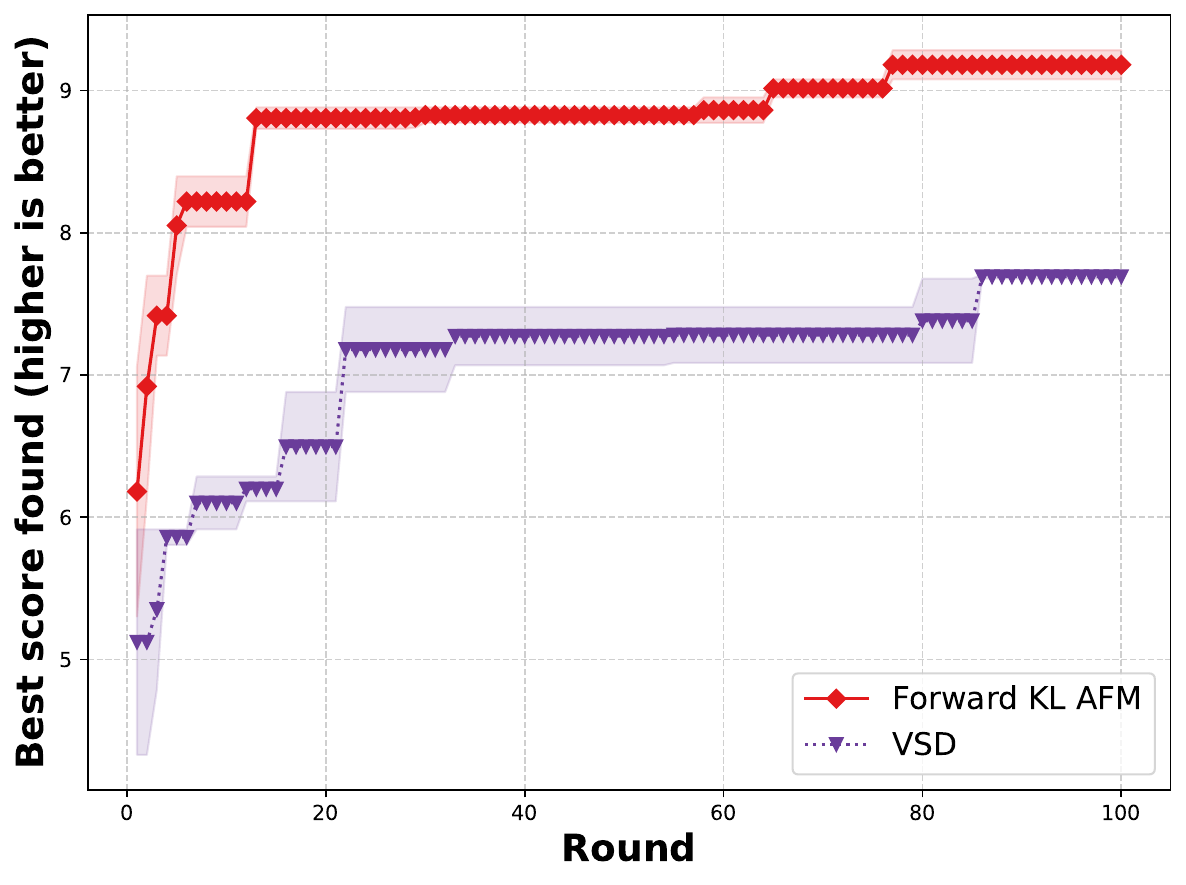}
        \caption{F2 Docking.}
        \label{fig:f2_docking}
    \end{subfigure}
    \caption{
        Structure-based protein design with FoldX oracles.
        We plot the best score found vs.\ oracle calls;
    }
    \label{fig:foldx_experiments}
\end{figure*}

\subsection{Baselines and Implementation}

We compare three AFM variants (forward-KL, reverse-KL, symmetric-KL) against three baselines: VSD \citep{Steinberg2025VSD}, CbAS \citep{Brookes2019CbAS}, and LaMBO-2 \citep{Gruver2023LaMBO2}. VSD minimises the reverse-KL divergence based objective, while CbAS minimises the forward-KL; both use autoregressive models with tractable likelihoods. LaMBO-2 uses discrete diffusion with gradient-based guidance in the latent space of the denoising network to guide generation toward high-fitness regions. For VSD and CbAS, we use the public implementation of \citet{Steinberg2025VSD}\footnote{\url{https://github.com/csiro-funml/variationalsearch}} with transformer backbones, which they showed to outperform LSTM counterparts. For LaMBO-2, we use the official implementation provided in the \textsc{poli} library.\footnote{\url{https://github.com/prescient-design/cortex}} For the mutation-constrained setting, LaMBO-2 provides built-in functionality to control the number of edits during generation. For VSD and CbAS, we use a masked-transformer (mTFM) backbone following \citep{Steinberg2025AGPS}, which naturally supports edit-constrained generation. The hyperparameter choices for AFM 
are provided in Appendix~\ref{app:implementation}.




\subsection{Results}

On both Ehrlich-32 and Ehrlich-64, forward-KL AFM converges fastest to near-optimal solutions, with symmetric-KL AFM following closely behind, while reverse-KL AFM lags behind both variants. VSD exhibits slower convergence, with the gap widening on longer sequences (Ehrlich-64), suggesting difficulty capturing long-range epistatic interactions. LaMBO-2 also struggles with slower convergence across both sequence lengths. CbAS shows strong exploitative behavior, failing to explore beyond initially discovered sequences and converging prematurely to suboptimal regions. On the AAV capsid design task, both forward-KL and symmetric-KL AFM perform strongly, achieving the lowest regret, with VSD following closely behind. Reverse-KL AFM struggles on this task, exhibiting poorer convergence in comparison.
Based on the relatively weaker performance of reverse-KL and symmetric-KL AFM on the Ehrlich and AAV tasks, we focus on forward-KL AFM for the remaining experiments. On the stability optimisation task, forward-KL AFM performs best, discovering high-stability protein variants more quickly than other methods. However, on the SASA optimisation task, forward-KL AFM struggles and underperforms both VSD and CbAS.
On the F2 molecular docking task, forward-KL AFM strongly outperforms VSD, achieving substantially higher docking scores throughout optimisation.

\section{Limitations}

Active Flow Matching relies on accurate estimates of the high-fitness probability $p(y \ge \tau \mid \rvx)$, and its performance is therefore sensitive to classifier quality, particularly in early rounds when labelled data are sparse. While self-normalised importance sampling yields consistent estimators asymptotically, finite-sample variance can be high when the proposal distribution poorly overlaps with the target, necessitating careful proposal design and mixing. In addition, our reverse-KL formulation lacks the same theoretical consistency guarantees as the forward-KL variant, and empirically exhibits mode-seeking behaviour that can lead to premature convergence.

\section{Conclusion}

We introduced Active Flow Matching, a principled framework for integrating implicit discrete flow models with variational active generation objectives. By reformulating KL-based optimisation criteria to operate on tractable conditional endpoint distributions along the flow, AFM enables fine-tuning of discrete flows toward high-fitness regions without requiring access to intractable marginal likelihoods. Our forward-KL formulation admits a consistency guarantee and empirically delivers strong exploration–exploitation trade-offs, outperforming state-of-the-art baselines across diverse protein design benchmarks under tight evaluation budgets.

More broadly, AFM bridges a gap between modern non-autoregressive generative models and probabilistically grounded black-box optimisation frameworks such as VSD and CbAS. 
This provides an additional approach to attack the problem of active generation, as specified in~\cite{Steinberg2025VSD}.
This connection opens several promising directions, including tighter variance reduction strategies, extensions to multi-objective and constrained optimisation, and applications beyond biological sequence design. We view Active Flow Matching as a step toward unifying expressive implicit generative models with rigorous decision-theoretic objectives for black-box optimisation in discrete, high-dimensional spaces.

\bibliography{references}

@article{Starr2016EpistasisProtein,
  author    = {Tyler N. Starr and Joseph W. Thornton},
  title     = {Epistasis in protein evolution},
  journal   = {Protein Science},
  volume    = {25},
  number    = {7},
  pages     = {1204--1218},
  year      = {2016},
}

@article{Phillips2008Epistasis,
  author    = {Patrick C. Phillips},
  title     = {Epistasis --- the essential role of gene interactions in the structure and evolution of genetic systems},
  journal   = {Nature Reviews Genetics},
  volume    = {9},
  number    = {11},
  pages     = {855--867},
  year      = {2008},
}

@article{Austin2021D3PM,
  author    = {Jacob Austin and Daniel D. Johnson and Jonathan Ho and Daniel Tarlow and Rianne van den Berg},
  title     = {Structured Denoising Diffusion Models in Discrete State Spaces},
  journal   = {arXiv},
  eprint    = {2107.03006},
  primaryClass = {cs.LG},
  year      = {2021},
}

@inproceedings{Gat2024DFM,
  author    = {Itai Gat and Tal Remez and Neta Shaul and Felix Kreuk and Ricky T. Q. Chen and Gabriel Synnaeve and Yossi Adi and Yaron Lipman},
  title     = {Discrete Flow Matching},
  booktitle = {Advances in Neural Information Processing Systems},
  year      = {2024},
}

@article{Watson2023RFdiffusion,
  author    = {Joseph L. Watson and David Juergens and Nathaniel R. Bennett and Brian L. Trippe and Jason Yim and Helen E. Eisenach and Woody Ahern and Andrew J. Borst and Robert J. Ragotte and Lukas F. Milles and Basile I. M. Wicky and Nikita Hanikel and Samuel J. Pellock and Alexis Courbet and William Sheffler and Jue Wang and Preetham Venkatesh and Isaac Sappington and Susana V{\'a}zquez Torres and Anna Lauko and Valentin De Bortoli and Emile Mathieu and Sergey Ovchinnikov and Regina Barzilay and Tommi S. Jaakkola and Frank DiMaio and Minkyung Baek and David Baker},
  title     = {De novo design of protein structure and function with {RFdiffusion}},
  journal   = {Nature},
  volume    = {620},
  number    = {7976},
  pages     = {1089--1100},
  year      = {2023},
}

@article{Ingraham2023Chroma,
  author    = {John B. Ingraham and Max Baranov and Zak Costello and Karl W. Barber and Wujie Wang and Ahmed Ismail and Vincent Frappier and Dana M. Lord and Christopher Ng-Thow-Hing and Erik R. Van Vlack and Shan Tie and Vincent Xue and Sarah C. Cowles and Alan Leung and Jo{\~a}o V. Rodrigues and Claudio L. Morales-Perez and Alex M. Ayoub and Robin Green and Katherine Puentes and Frank Oplinger and Nishant V. Panwar and Fritz Obermeyer and Adam R. Root and Andrew L. Beam and Frank J. Poelwijk and Gevorg Grigoryan},
  title     = {Illuminating protein space with a programmable generative model},
  journal   = {Nature},
  volume    = {623},
  number    = {7988},
  pages     = {1070--1078},
  year      = {2023},
}

@inproceedings{Steinberg2025VSD,
  author = {Daniel M. Steinberg and Rafael Oliveira and Cheng Soon Ong and Edwin V. Bonilla},
  title = {Variational Search Distributions},
  year = {2025},
  booktitle = {International Conference on Learning Representations},
}

@misc{Steinberg2025AGPS,
  author = {Daniel M. Steinberg and Asiri Wijesinghe and Rafael Oliveira and Piotr Koniusz and Cheng Soon Ong and Edwin V. Bonilla},
  title = {Amortized Active Generation of Pareto Sets},
  year = {2025},
  howpublished = {arXiv:2510.21052},
}

@misc{ScrippsSPRFees,
  author = {{Scripps Research} Biophysics and Biochemistry Core},
  title = {Surface Plasmon Resonance ({SPR}) Fees},
  year = {2024},
}

@misc{DukeBIAFees,
  author = {{Duke University} Biomolecular Interaction Analysis Core},
  title = {Reservations, Policies, and Rates},
  year = {2024},
  note = {SPR/BLI hourly rates; sample prep/data analysis fees},
}

@inproceedings{Brookes2019CbAS,
  author    = {David H. Brookes and Hahnbeom Park and Jennifer Listgarten},
  title     = {Conditioning by adaptive sampling for robust design},
  booktitle = {International Conference on Machine Learning},
  pages     = {773--782},
  year      = {2019},
}

@article{Lipman2022FlowMatching,
  author    = {Yaron Lipman and Ricky T. Q. Chen and Heli Ben-Hamu and Maximilian Nickel and Matt Le},
  title     = {Flow Matching for Generative Modeling},
  journal   = {arXiv},
  eprint    = {2210.02747},
  year      = {2022},
}

@inproceedings{Gruver2023LaMBO2,
  author    = {Nate Gruver and Samuel Stanton and Nathan C. Frey and Tim G. J. Rudner and Isidro Hotzel and Julien Lafrance-Vanasse and Arvind Rajpal and Kyunghyun Cho and Andrew Gordon Wilson},
  title     = {Protein Design with Guided Discrete Diffusion},
  booktitle = {Advances in Neural Information Processing Systems},
  year      = {2023},
}

@misc{Stanton2024Ehrlich,
  title  = {Closed-Form Test Functions for Biophysical Sequence Optimization Algorithms},
  author = {Samuel Stanton and Robert Alberstein and Nathan Frey and Andrew Watkins and Kyunghyun Cho},
  year   = {2024},
  eprint = {2407.00236},
  archivePrefix = {arXiv},
  primaryClass  = {cs.LG},
  doi    = {10.48550/arXiv.2407.00236}
}

@article{Bryant2021AAV,
  title   = {Deep Diversification of an AAV Capsid Protein by Machine Learning},
  author  = {Drew H. Bryant and Ali Bashir and Sam Sinai and Nina K. Jain and Pierce J. Ogden and Patrick F. Riley and George M. Church and Lucy J. Colwell and Eric D. Kelsic},
  journal = {Nature Biotechnology},
  volume  = {39},
  pages   = {691--696},
  year    = {2021},
  doi     = {10.1038/s41587-020-00793-4}
}

@article{Schymkowitz2005FoldXWeb,
  title   = {The FoldX Web Server: An Online Force Field},
  author  = {Joost Schymkowitz and Jesper Borg and Fran{\c{c}}ois Stricher and Robby Nys and Fr{\'e}d{\'e}ric Rousseau and Luis Serrano},
  journal = {Nucleic Acids Research},
  volume  = {33},
  pages   = {W382--W388},
  year    = {2005},
}

@article{campbell2024generative,
  title={Generative flows on discrete state-spaces: Enabling multimodal flows with applications to protein co-design},
  author={Campbell, Andrew and Yim, Jason and Barzilay, Regina and Rainforth, Tom and Jaakkola, Tommi},
  journal={arXiv preprint arXiv:2402.04997},
  year={2024}
}

@article{campbell2022continuous,
  title={A continuous time framework for discrete denoising models},
  author={Campbell, Andrew and Benton, Joe and De Bortoli, Valentin and Rainforth, Thomas and Deligiannidis, George and Doucet, Arnaud},
  journal={Advances in Neural Information Processing Systems},
  volume={35},
  pages={28266--28279},
  year={2022}
}

@article{pooladian2023multisample,
  title={Multisample flow matching: Straightening flows with minibatch couplings},
  author={Pooladian, Aram-Alexandre and Ben-Hamu, Heli and Domingo-Enrich, Carles and Amos, Brandon and Lipman, Yaron and Chen, Ricky TQ},
  journal={arXiv preprint arXiv:2304.14772},
  year={2023}
}

@inproceedings{Kirjner2024GGS,
  title={Improving protein optimization with smoothed fitness landscapes},
  author={Kirjner, Andrew and Yim, Jason and Samusevich, Raman and Bracha, Shahar and Jaakkola, Tommi S. and Barzilay, Regina and Fiete, Ila R.},
  booktitle={International Conference on Learning Representations},
  year={2024}
}

@inproceedings{gonzalez2024poli,
  title={Poli: A Library of Discrete Sequence Objectives},
  author={Gonz{\'a}lez-Duque, Miguel and Bartels, Simon and Michael, Richard},
  booktitle={Advances in Neural Information Processing Systems},
  volume={37},
  year={2024},
}

@misc{lipman2024flowmatchingguidecode,
      title={Flow Matching Guide and Code}, 
      author={Yaron Lipman and Marton Havasi and Peter Holderrieth and Neta Shaul and Matt Le and Brian Karrer and Ricky T. Q. Chen and David Lopez-Paz and Heli Ben-Hamu and Itai Gat},
      year={2024},
      eprint={2412.06264},
      archivePrefix={arXiv},
}

@article{Nijkamp2023ProGen2,
  title={ProGen2: Exploring the Boundaries of Protein Language Models},
  author={Nijkamp, Erik and Ruffolo, Jeffrey A and Weinstein, Eli N and Naber, Christopher and Madani, Ali},
  journal={Cell Systems},
  volume={14},
  number={11},
  pages={968--978},
  year={2023},
  publisher={Elsevier}
}

@inproceedings{Notin2022Tranception,
  title={Tranception: Protein Fitness Prediction with Autoregressive Transformers and Inference-time Retrieval},
  author={Notin, Pascal and Dias, Mafalda and Frazer, Jonathan and Marchena-Hurtado, Javier and Gomez, Aidan and Marks, Debora and Gal, Yarin},
  booktitle={International Conference on Machine Learning},
  pages={16990--17017},
  year={2022},
  organization={PMLR}
}

@inproceedings{Sahoo2024MDLM,
  title={Simple and Effective Masked Diffusion Language Models},
  author={Sahoo, Subham Sekhar and Arriola, Marianne and Schiff, Yair and Gokaslan, Aaron and Marroquin, Edgar and Chiu, Justin T and Rush, Alexander and Kuleshov, Volodymyr},
  booktitle={Advances in Neural Information Processing Systems},
  volume={37},
  year={2024}
}

@article{Alamdari2023EvoDiff,
  title={Protein Generation with Evolutionary Diffusion: {S}equence is All You Need},
  author={Alamdari, Sarah and Thakur, Nitya and van den Berg, Rianne and Lu, Alex X and Fusi, Nicolo and Louber, Ava P and Madani, Ali},
  journal={bioRxiv},
  pages={2023.09.11.556673},
  year={2023},
  publisher={Cold Spring Harbor Laboratory}
}

@inproceedings{Dhariwal2021Classifier,
  title={Diffusion Models Beat {GAN}s on Image Synthesis},
  author={Dhariwal, Prafulla and Nichol, Alexander},
  booktitle={Advances in Neural Information Processing Systems},
  volume={34},
  pages={8780--8794},
  year={2021}
}

@article{Ho2022ClassifierFree,
  title={Classifier-Free Diffusion Guidance},
  author={Ho, Jonathan and Salimans, Tim},
  journal={arXiv preprint arXiv:2207.12598},
  year={2022}
}

@inproceedings{Jang2017GumbelSoftmax,
  title={Categorical Reparameterization with {G}umbel-Softmax},
  author={Jang, Eric and Gu, Shixiang and Poole, Ben},
  booktitle={International Conference on Learning Representations},
  year={2017}
}

@misc{Bengio2013StraightThrough,
      title={Estimating or Propagating Gradients Through Stochastic Neurons for Conditional Computation}, 
      author={Yoshua Bengio and Nicholas Léonard and Aaron Courville},
      year={2013},
      eprint={1308.3432},
      archivePrefix={arXiv},
      primaryClass={cs.LG},
}

@inproceedings{Nisonoff2024DiscreteGuidance,
  title={Unlocking Guidance for Discrete State-Space Diffusion and Flow Models},
  author={Nisonoff, Hunter and Xiong, Junhao and Alber, Stephan and Listgarten, Jennifer},
  booktitle={International Conference on Learning Representations},
  year={2025}
}

@article{Yang2019MLDE,
  title={Machine-Learning-Guided Directed Evolution for Protein Engineering},
  author={Yang, Kevin K and Wu, Zachary and Arnold, Frances H},
  journal={Nature Methods},
  volume={16},
  number={8},
  pages={687--694},
  year={2019},
  publisher={Nature Publishing Group}
}

@article{garciaortegon2022dockstring,
  title={{DOCKSTRING}: Easy Molecular Docking Yields Better Benchmarks for Ligand Design},
  author={Garc{\'\i}a-Orteg{\'o}n, Miguel and Simm, Gregor NC and Tripp, Austin J and Hern{\'a}ndez-Lobato, Jos{\'e} Miguel and Bender, Andreas and Bacallado, Sergio},
  journal={J. Chem. Inf. Model.},
  volume={62},
  number={15},
  pages={3486--3502},
  year={2022}
}

@article{krenn2020selfies,
  title={Self-Referencing Embedded Strings ({SELFIES}): A 100\% Robust Molecular String Representation},
  author={Krenn, Mario and H{\"a}se, Florian and Nigam, AkshatKumar and Friederich, Pascal and Aspuru-Guzik, Al{\'a}n},
  journal={Mach. Learn.: Sci. Technol.},
  volume={1},
  number={4},
  pages={045024},
  year={2020}
}

@article{bickerton2012quantifying,
  title={Quantifying the Chemical Beauty of Drugs},
  author={Bickerton, G Richard and Paolini, Gaia V and Besnard, J{\'e}r{\'e}my and Muresan, Sorel and Hopkins, Andrew L},
  journal={Nat. Chem.},
  volume={4},
  number={2},
  pages={90--98},
  year={2012}
}

@article{brookes2018design,
  title={Design by adaptive sampling},
  author={Brookes, David H and Listgarten, Jennifer},
  journal={arXiv preprint arXiv:1810.03714},
  year={2018}
}
\bibliographystyle{icml2026}

\newpage
\appendix
\onecolumn
\section{Appendix.}


\subsection{Proof of Theorem~\ref{thm:fwd_consistency}}
\label{app:proofs}
\textbf{Theorem~\ref{thm:fwd_consistency}.} \textit{Let $p^*(\rvx) \propto p_{\mathrm{prior}}(\rvx) w(\rvx)$ be the target distribution, where $w(\rvx) = p(y \ge \tau \mid \rvx)$ and $p_{\mathrm{prior}}$ is the uniform distribution. Under standard DFM assumptions (masked-source coupling, convex interpolant paths, and strictly positive scheduler), the global minimiser $\phi^*$ of the Forward-KL AFM objective yields a terminal distribution $q_1^{\phi^*} = p^*$ almost everywhere.}

\begin{proof}
The proof proceeds in two stages: first, we establish that the Forward-KL AFM objective is equivalent to the standard Discrete Flow Matching (DFM) objective trained on the target distribution $p^*$. Second, we invoke existing consistency results for DFM to show that this leads to the correct terminal distribution.

\textbf{1. Objective Equivalence.}

The Forward-KL AFM objective (Eq.~\ref{eq:fwd_snis} in the main text) estimates the following expected loss over the prior distribution:
\begin{equation}
    \mathcal{L}_{\mathrm{AFM}}(\phi) = \mathbb{E}_{t \sim \mathcal{U}[0,1]} \mathbb{E}_{\rvx_1 \sim p_{\mathrm{prior}}} \left[ w(\rvx_1) \cdot \ell_\phi(\rvx_1, t) \right],
\end{equation}
where, under the masked-source coupling with $\rvx_0 = m$ deterministic,
\begin{equation}
    \ell_\phi(\rvx_1, t) = \mathbb{E}_{\rvx_t \sim p_t(\cdot \mid m, \rvx_1)} \left[-\log q_\phi(\rvx_1 \mid \rvx_t, t)\right]
\end{equation}
is the standard flow matching loss for a single datum $\rvx_1$ at time $t$. We can rewrite the outer expectation over $\rvx_1$ as:
\begin{equation}
    J(\phi) = \sum_{\rvx_1 \in \mathcal{X}} p_{\mathrm{prior}}(\rvx_1) w(\rvx_1) \ell_\phi(\rvx_1),
\end{equation}
where we have absorbed the time expectation into $\ell_\phi$.

By definition, the target distribution is $p^*(\rvx_1) = \frac{1}{Z} p_{\mathrm{prior}}(\rvx_1) w(\rvx_1)$, which implies the identity $p_{\mathrm{prior}}(\rvx_1) w(\rvx_1) = Z p^*(\rvx_1)$. Substituting this into the objective:
\begin{equation}
    J(\phi) = \sum_{\rvx_1 \in \mathcal{X}} Z p^*(\rvx_1) \ell_\phi(\rvx_1) = Z \cdot \mathbb{E}_{\rvx_1 \sim p^*} \left[ \ell_\phi(\rvx_1) \right].
\end{equation}

Since $Z = \mathbb{E}_{p_{\mathrm{prior}}}[w(\rvx)]$ is a constant independent of $\phi$, minimising $\mathcal{L}_{\mathrm{AFM}}(\phi)$ is equivalent to minimising the standard DFM objective $\mathcal{L}_{\mathrm{DFM}}(\phi; p^*) = \mathbb{E}_{\rvx_1 \sim p^*} [\ell_\phi(\rvx_1)]$.

\textit{Implementation Note:} In Algorithm~\ref{alg:afm_fwd}, we approximate this objective using a mixture proposal $\mu(\rvx_1)$ (Eq.~\ref{eq:mixture}). By using importance weights from a mixture distribution, the empirical estimator is a consistent estimator of $J(\phi)$, preserving the theoretical minimum derived here.

\textit{SNIS Consistency.} In practice, we use self-normalised importance sampling, which introduces finite-sample bias of order $O(1/K)$. However, SNIS is consistent: as $K \to \infty$, the estimator converges to the true expectation, and the minimiser of the empirical objective converges to $\phi^*$.

\textbf{2. Consistency of the Minimiser.}

Having established that we are effectively minimising the standard DFM objective for the target distribution $p^*$, the result follows from the consistency guarantees established in \citet{Gat2024DFM}:

\begin{itemize}
    \item Optimal Posterior Recovery (Prop.~5): The global minimiser $\phi^*$ of the cross-entropy loss $\mathcal{L}_{\mathrm{DFM}}(\phi; p^*)$ recovers the true token-level conditional posterior of the target distribution: $p_{\phi^*}(x_1^i \mid \rvx_t, t) = p^*(x_1^i \mid \rvx_t, t)$ for each token $i \in [N]$.
    
    \item Conditional Velocity (Thm.~3): For the convex interpolant path family, the probability velocity field constructed from this posterior via Eq.~\ref{eq:velocity} generates the conditional probability path $p_t(\cdot \mid m, \rvx_1)$.
    
    \item Marginal Velocity (Thm.~2): The marginal velocity field, obtained by averaging the conditional velocity under the posterior $p_t(\rvx_0, \rvx_1 \mid \rvx_t)$, generates the marginal probability path $p_t^*$.
    
    \item Terminal Distribution: Iterating the sampling rule $x_{t+h}^i \sim \delta_{\rvx_t^i}(\cdot) + h \, u_t^i(\cdot, \rvx_t)$ produces samples from $p_t^*$. By the boundary condition of the interpolant path, $p_1^* = p^*$, so the terminal distribution satisfies $q_1^{\phi^*} = p^*$.
\end{itemize}

Thus, the generated distribution matches the target $p^*$ almost everywhere.
\end{proof}
\subsection{Implementation Details}
\label{app:implementation}

\subsubsection{Ehrlich Task Hyperparameters}
We utilise the implementation provided by \citet{Stanton2024Ehrlich}. The tasks define a sequence length of $L=32$ (or $L=64$) with an alphabet size of $|\mathcal{V}|=20$.

\begin{table}[h]
    \centering
    \caption{Hyperparameters and dynamic scheduling for Ehrlich experiments.}
    \label{tab:ehrlich_hyperparams}
    \begin{tabular}{ll}
        \toprule
        \textbf{Parameter} & \textbf{Value / Schedule} \\
        \midrule
        \multicolumn{2}{c}{\textit{Generative Flow \& Classifier Architecture}} \\
        \midrule
        Architecture & Bi-LSTM \\
        Hidden Dimension & $256$ \\
        Embedding Dimension & $128$ \\
        Layers & $3$ \\
        Dropout & $0.1$ \\
        Time Scheduler & Quadratic ($t^2$) \\
        \midrule
        \multicolumn{2}{c}{\textit{Active Generation Loop (Round $r$)}} \\
        \midrule
        Gradient Steps per Round & $2{,}000$ \\
        Optimiser & AdamW \\
        Learning Rate & $1 \times 10^{-4} \times 0.95^r$ \\
        Threshold $\tau$ & $90^{\text{th}}$ percentile (dynamic) \\
        \midrule
        \multicolumn{2}{c}{\textit{Mixing Coefficients (Round $r$)}} \\
        \midrule
        Replay ($\alpha_{\mathrm{rb}}$) & $\min(0.4, 0.05 \times r)$ \\
        Flow ($\alpha_{\mathrm{flow}}$) & $(1 - \alpha_{\mathrm{rb}}) \times 0.1$ (early) \\
                                        & $(1 - \alpha_{\mathrm{rb}}) \times 0.95$ (late) \\
        Prior ($\alpha_0$) & $1 - \alpha_{\mathrm{rb}} - \alpha_{\mathrm{flow}}$ \\
        Buffer Temp ($\gamma$) & $0.3$ \\
        \bottomrule
    \end{tabular}
\end{table}

\paragraph{Proposal Schedule.}
The importance sampling proposal $\mu(\rvx_1)$ is adjusted dynamically across two regimes. In early rounds, when the replay buffer is small, we sample primarily from the prior to encourage broad exploration. Once the replay buffer exceeds the batch size, the flow component dominates ($\alpha_{\mathrm{flow}} \approx 95\%$ of non-replay probability), enabling local refinement around promising regions. Throughout training, the replay buffer contribution ($\alpha_{\mathrm{rb}}$) increases linearly from $0\%$ to $40\%$, progressively shifting focus toward exploitation of confirmed high-fitness sequences.

\subsection{FoldX Task Hyperparameters}

The FoldX task uses the same generative flow and classifier architecture as Ehrlich (Table~\ref{tab:ehrlich_hyperparams}), with sequence length $L=228$ and alphabet size $|\mathcal{V}|=21$ (20 amino acids plus a padding token).

\begin{table}[h]
    \centering
    \caption{Hyperparameters for FoldX experiments. Unlisted parameters match Table~\ref{tab:ehrlich_hyperparams}.}
    \label{tab:foldx_hyperparams}
    \begin{tabular}{ll}
        \toprule
        \textbf{Parameter} & \textbf{Value / Schedule} \\
        \midrule
        \multicolumn{2}{c}{\textit{Mixing Coefficients (Round $r$)}} \\
        \midrule
        Replay ($\alpha_{\mathrm{rb}}$) & $0.3$ (fixed) \\
        Flow ($\alpha_{\mathrm{flow}}$) & $(1 - \alpha_{\mathrm{rb}}) \times 0.15$ ($r < 5$) \\
                                        & $(1 - \alpha_{\mathrm{rb}}) \times 0.35$ ($r \geq 5$) \\
        Prior ($\alpha_0$) & $1 - \alpha_{\mathrm{rb}} - \alpha_{\mathrm{flow}}$ \\
        Buffer Temp ($\gamma$) & $0.2$ \\
        Mutation Budget $K$ & $3$ \\
        \bottomrule
    \end{tabular}
\end{table}

\paragraph{Source Distribution.}
Unlike the Ehrlich task, in the foldx task, we restrict the number of mutations that the flow can make to the source sequence. The flow source $x_0$ is sampled from the replay buffer rather than a uniform prior. This focuses training on local refinement around known high-fitness regions.

\paragraph{Constrained Sampling.}
We enforce a mutation budget $K \leq 3$ by modifying the Euler sampler to track edits from a reference sequence $x_{\mathrm{ref}}$. At each step, we sample position and destination token proportionally to the predicted rates, masking self-transitions and padding tokens. The process terminates after $K$ mutations or when no favourable transitions remain.

\subsection{AAV Capsid Design Hyperparameters}

For the AAV task ($L=28$, $|\mathcal{V}|=20$), we utilise Transformer-based architectures for both the generative flow and classifier. The ground-truth fitness is evaluated using a pre-trained Convolutional Neural Network (CNN) oracle.

\begin{table}[h]
    \centering
    \caption{Hyperparameters for AAV Capsid Optimisation.}
    \label{tab:aav_hyperparams}
    \begin{tabular}{ll}
        \toprule
        \textbf{Parameter} & \textbf{Value / Schedule} \\
        \midrule
        \multicolumn{2}{c}{\textit{Generative Flow Model}} \\
        \midrule
        Architecture & Transformer \\
        Layers & $5$ \\
        Heads & $8$ \\
        Hidden Dim & $256$ \\
        Dropout & $0.1$ \\
        \midrule
        \multicolumn{2}{c}{\textit{Classifier Model}} \\
        \midrule
        Architecture & Transformer Encoder \\
        Layers & $3$ \\
        Heads & $8$ \\
        Embedding Dim & $128$ \\
        Feedforward Dim & $512$ \\
        Dropout & $0.25$ \\
        \midrule
        \multicolumn{2}{c}{\textit{Active Generation Loop (Round $r$)}} \\
        \midrule
        Gradient Steps & $1{,}000$ \\
        Optimiser & AdamW \\
        Learning Rate & $1 \times 10^{-4} \times 0.95^r$ \\
        Threshold $\tau$ & $13.5$ (initial), $14.0$ ($r \geq 2$) \\
                         & $15.0$ ($r \geq 5$), $15.5$ ($r \geq 6$) \\
        \midrule
        \multicolumn{2}{c}{\textit{Mixing Coefficients (Round $r$)}} \\
        \midrule
        Replay ($\alpha_{\mathrm{rb}}$) & $0.3$ (fixed) \\
        Flow ($\alpha_{\mathrm{flow}}$) & $(1 - \alpha_{\mathrm{rb}}) \times 0.35$ (early) \\
                                        & $(1 - \alpha_{\mathrm{rb}}) \times 0.95$ (late) \\
        Prior ($\alpha_0$) & $1 - \alpha_{\mathrm{rb}} - \alpha_{\mathrm{flow}}$ \\
        Buffer Temp ($\gamma$) & $0.4$ \\
        \bottomrule
    \end{tabular}
\end{table}

\paragraph{Oracle Details.}
The fitness landscape is defined by a CNN oracle trained on the dataset from \citet{Bryant2021AAV}. Scores are denormalised to the original scale for evaluation.

\paragraph{Proposal Strategy.}
Unlike the Ehrlich task, the AAV experiment uses a fixed replay ratio ($\alpha_{\mathrm{rb}} = 0.3$) throughout optimisation. As with Ehrlich, the flow component dominates once the replay buffer accumulates sufficient samples ($\alpha_{\mathrm{flow}} \approx 66\%$ of non-replay probability).

\subsection{Dockstring Molecular Docking Hyperparameters}

For the Dockstring task, molecules are represented as SELFIES strings and vocabulary size $|\mathcal{V}| \approx 100$ (varies by dataset). We use an LSTM-based architecture for the generative flow and an LSTM classifier for the CPE.

\begin{table}[h]
    \centering
    \caption{Hyperparameters for Dockstring (F2/Thrombin) experiments.}
    \label{tab:dockstring_hyperparams}
    \begin{tabular}{ll}
        \toprule
        \textbf{Parameter} & \textbf{Value / Schedule} \\
        \midrule
        \multicolumn{2}{c}{\textit{Generative Flow Model}} \\
        \midrule
        Architecture & Bi-LSTM \\
        Hidden Dimension & $512$ \\
        Embedding Dimension & $256$ \\
        Layers & $4$ \\
        Dropout & $0.1$ \\
        Time Scheduler & Quadratic ($t^2$) \\
        \midrule
        \multicolumn{2}{c}{\textit{Classifier (CPE) Model}} \\
        \midrule
        Architecture & Bi-LSTM \\
        Hidden Dimension & $256$ \\
        Embedding Dimension & $128$ \\
        Layers & $3$ \\
        Dropout & $0.25$ \\
        \midrule
        \multicolumn{2}{c}{\textit{Active Generation Loop (Round $r$)}} \\
        \midrule
        Pretraining Steps & $5{,}000$ \\
        Gradient Steps per Round & $1{,}000$ \\
        Optimiser & AdamW \\
        Learning Rate (Pretrain) & $1 \times 10^{-3}$ \\
        Learning Rate (IW) & $1 \times 10^{-4}$ \\
        \midrule
        \multicolumn{2}{c}{\textit{Threshold Schedule}} \\
        \midrule
        Initial $\tau$ & $80^{\text{th}}$ percentile \\
        Final $\tau$ & $95^{\text{th}}$ percentile \\
        \midrule
        \multicolumn{2}{c}{\textit{Mixing Coefficients (Round $r$, $R=100$)}} \\
        \midrule
        Replay ($\alpha_{\mathrm{rb}}$) & $0.95 \to 0.70$ (linear) \\
        Flow ($\alpha_{\mathrm{flow}}$) & $0.10 \to 0.40$ (linear) \\
        Prior ($\alpha_0$) & $0.0$ (fixed) \\
        Buffer Temp ($\gamma$) & $0.5$ \\
        Buffer Size & $1{,}500$ \\
        \bottomrule
    \end{tabular}
\end{table}

\paragraph{Objective Function.}
Following~\citet{garciaortegon2022dockstring}, we optimise docking score with a QED penalty: $f(\rvx) = s_{\text{dock}}(\rvx) - \lambda(1 - \text{QED}(\rvx))$ where $\lambda = 7.5$. 

\paragraph{Proposal Schedule.}
Unlike the protein tasks, the molecular docking task uses a replay-heavy strategy throughout optimisation. The replay buffer is seeded with the top 2,500 molecules from the initial pool. The replay contribution decreases from $95\%$ to $70\%$ over 100 rounds, while the flow contribution increases correspondingly from $10\%$ to $40\%$. This conservative schedule reflects the larger search space and sparser reward landscape of molecular optimisation.


\end{document}